\pdfoutput=1

\documentclass[11pt]{article}

\usepackage[colorlinks=true,allcolors=blue]{hyperref}
\usepackage{url}
\usepackage{graphicx}
\usepackage{booktabs}
\usepackage{multirow}
\usepackage{amsmath,amssymb,amsthm}
\usepackage{wrapfig}
\usepackage{enumitem}
\usepackage{natbib}
\usepackage{tikz}
\usepackage{bm}
\usetikzlibrary{shapes,arrows,arrows.meta,positioning}

\usepackage{algorithm}
\usepackage{algorithmic}

\usepackage{import}
\usepackage{standalone}
\usepackage{pgfplots}
\usepgfplotslibrary{fillbetween}
\usepgfplotslibrary{groupplots}
\pgfplotsset{compat=1.18}

\theoremstyle{plain}
\newtheorem{theorem}{Theorem}[section]
\newtheorem{proposition}[theorem]{Proposition}

\theoremstyle{definition}
\newtheorem{definition}[theorem]{Definition}

\theoremstyle{remark}

\newcommand{\diag}{\mathrm{diag}}
\newcommand{\HFER}{\mathrm{HFER}}
\newcommand{\SE}{\mathrm{SE}}
\newcommand{\R}{\mathbb{R}}

\newcommand{\Tr}{\mathrm{Tr}}

\title{Spectral Guardrails for Agents in the Wild: \\ Detecting Tool Use Hallucinations via Attention Topology}

\author{Valentin Noël \\
Devoteam, Paris, France \\
\texttt{valentin.noel@devoteam.com}
}
\date{}

\begin{document}

\maketitle

\begin{abstract}
Deploying autonomous agents in the wild requires reliable safeguards against tool use failures. We propose a training free guardrail based on spectral analysis of attention topology that complements supervised approaches. On Llama 3.1 8B, our method achieves 97.7\% recall with multi-feature detection and 86.1\% recall with 81.0\% precision for balanced deployment, without requiring any labeled training data. Most remarkably, we discover that single layer spectral features act as near-perfect hallucination detectors: Llama L26 Smoothness achieves 98.2\% recall (213/217 hallucinations caught) with a single threshold, and Mistral L3 Entropy achieves 94.7\% recall. This suggests hallucination is not merely a wrong token but a thermodynamic state change: the model's attention becomes noise when it errs. Through controlled cross-model evaluation on matched domains ($N=1000$, $T=0.3$, same General domain, hallucination rates 20--22\%), we reveal the ``Loud Liar'' phenomenon: Llama 3.1 8B's failures are spectrally catastrophic and dramatically easier to detect, while Mistral 7B achieves the best discrimination (AUC 0.900). These findings establish spectral analysis as a principled, efficient framework for agent safety.
\end{abstract}

\section{Introduction}

As large language model agents move from research demonstrations to real world deployment, reliability becomes paramount. An agent tasked with financial transactions, appointment scheduling, or system administration must invoke tools correctly \citep{patil2024gorilla, qin2023toolllm}; hallucinated function names, fabricated parameters, or schema violations can cause direct harm to users. For safety critical applications, the cost of a missed hallucination far exceeds the cost of a false alarm, making recall the primary optimization target.

Current approaches to hallucination detection rely on supervised probing \citep{azaria2023internal, chen2024inside, zhang2025prompt, zhang2025detecting, bhatnagar2026halt, hou2025probabilistic}. Recent work by \citet{healy2026internal} trains multi layer perceptrons on hidden states, achieving strong performance on tool calling tasks. While effective, supervised approaches require labeled samples per model, which can be limiting when agents continuously adapt to new tools or when deployment contexts shift.

We propose a complementary approach: spectral analysis of attention topology. By treating attention matrices as dynamic graphs \citep{el2025towards, noel2025gsp} and computing properties of their Laplacian spectrum \citep{chung1997spectral, shuman2013emerging}, we detect hallucinations without any learned parameters. Our method achieves 97.7\% recall on Llama 3.1 8B with multi-feature detection, and 86.1\% recall with 81.0\% precision for balanced deployment, demonstrating that training free spectral features can serve as effective guardrails \citep{binkowski2025hallucination, li2025llm}. The two approaches offer different tradeoffs: supervised probes may achieve higher precision when abundant labeled data is available, while spectral methods provide immediate deployability without data collection.

Our most striking finding is that single layer spectral features act as near perfect hallucination detectors. Llama L26 Smoothness achieves 98.2\% recall, catching 213 of 217 hallucinations with a single threshold. Mistral L3 Entropy \citep{jiang2023mistral} achieves 94.7\% recall. This suggests a deeper insight: hallucination is not merely selecting a wrong token; it is a thermodynamic state change in which the model's attention becomes noise. When a model hallucinates, its spectral energy disperses across modes rather than concentrating in coherent patterns \citep{shuman2013emerging}.

\begin{figure*}[ht]
    \centering
    \begin{tikzpicture}[
    node distance=1.2cm,
    block/.style={
        rectangle, draw, fill=white,
        inner sep=5pt, text width=1.9cm, align=center,
        rounded corners=3pt, font=\small\sffamily,
        line width=0.6pt
    },
    smallblock/.style={
        rectangle, draw, fill=white,
        inner sep=4pt, text width=1.8cm, align=center,
        rounded corners=2pt, font=\scriptsize\sffamily,
        line width=0.5pt
    },
    arrow/.style={-{Stealth[scale=1.1]}, thick, draw=gray!70}]
    
    \node (input) [block, fill=gray!10] {
        \textbf{Tool Call}\\[2pt]
        {\tiny\ttfamily transfer\_funds(}\\
        {\tiny\ttfamily \ \ acc="1234",}\\
        {\tiny\ttfamily \ \ amt=500)}
    };
    
    \node (transformer) [right=0.35cm of input, smallblock, fill=orange!10] {
        \textbf{Transformer}\\[2pt]
        {\tiny Extract $\mathbf{A}^{(\ell,h)}$, $\mathbf{X}^{(\ell)}$}
    };
    
    \node (laplacian) [right=0.35cm of transformer, block] {
        \textbf{Graph Laplacian}\\[2pt]
        {\tiny $\mathbf{L} = \mathbf{D} - \frac{1}{2}(\mathbf{A} + \mathbf{A}^\top)$}
    };
    
    \node (valid) [right=0.6cm of laplacian, smallblock, fill=blue!15, draw=blue!50] {
        \textbf{Valid Call}\\[3pt]
        \begin{tikzpicture}[scale=0.4]
            \foreach \i in {1,...,5} {
                \node[circle, fill=blue!60, inner sep=1.2pt] (n\i) at ({72*\i}:0.6) {};
            }
            \draw[blue!60, thick] (n1)--(n2)--(n3)--(n4)--(n5)--(n1);
            \draw[blue!40] (n1)--(n3) (n2)--(n4) (n3)--(n5);
        \end{tikzpicture}\\[-1pt]
        {\tiny Low Entropy}
    };
    
    \node (invalid) [below=0.5cm of valid, smallblock, fill=red!15, draw=red!50] {
        \textbf{Hallucination}\\[3pt]
        \begin{tikzpicture}[scale=0.4]
            \foreach \i in {1,...,5} {
                \node[circle, fill=red!60, inner sep=1.2pt] (n\i) at ({72*\i}:0.6) {};
            }
            \draw[red!40, thin] (n1)--(n2) (n4)--(n5);
        \end{tikzpicture}\\[-1pt]
        {\tiny High Entropy}
    };
    
    \node (output) [right=0.6cm of valid, yshift=-0.8cm, block, fill=green!10, draw=green!50!black] {
        \textbf{Guardrail}\\[2pt]
        {\scriptsize 98.2\% Recall}\\
        {\scriptsize (single feature)}
    };
    
    \draw[arrow] (input) -- (transformer);
    \draw[arrow] (transformer) -- (laplacian);
    \draw[arrow] (laplacian.east) -- ++(0.4,0) |- (valid.west);
    \draw[arrow] (laplacian.east) -- ++(0.4,0) |- (invalid.west);
    \draw[arrow] (valid.east) -- ++(0.3,0) |- (output.west);
    \draw[arrow] (invalid.east) -- ++(0.3,0) |- (output.west);
    
    \node[above=0.05cm of valid, font=\tiny\itshape, text=blue!70] {Coherent topology};
    \node[below=0.05cm of invalid, font=\tiny\itshape, text=red!70] {Entropic collapse};
    
    \end{tikzpicture}
    \caption{\textbf{Method Overview.} Spectral analysis of attention graphs enables training free hallucination detection. Hallucinations manifest as spectral collapse, a thermodynamic signature of incoherent reasoning. A single Smoothness feature achieves up to 98.2\% recall on Llama.}
    \label{fig:pipeline_teaser}
\end{figure*}
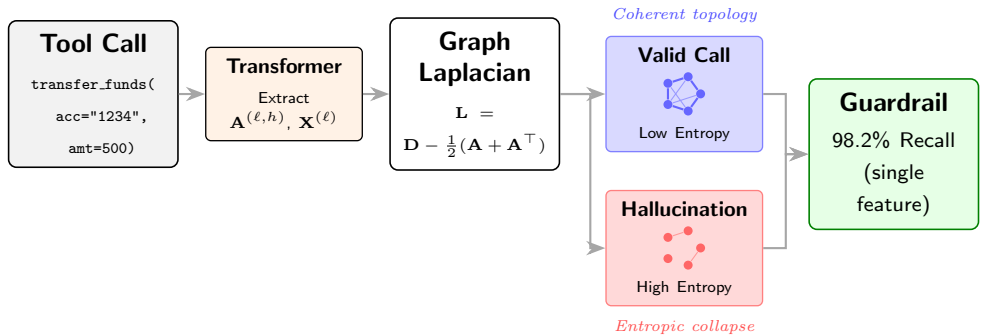

A critical methodological contribution of this work is our controlled cross-model analysis. By evaluating all three architectures on the same General domain with matched hallucination rates (Qwen 20.0\%, Mistral 20.7\%, Llama 21.7\%), we isolate architecture-specific failure signatures independent of domain or base rate confounds. This reveals what we call the ``Loud Liar'' phenomenon: Llama 3.1 8B's hallucinations are spectrally catastrophic, with a single feature catching 98.2\% of failures. In contrast, Mistral achieves the best discriminative ability (AUC 0.900), maintaining a sharper boundary between valid and invalid calls. These are fundamentally different failure geometries with distinct implications for deployment.

We additionally evaluate domain-specific effects, finding that Llama hallucinates 3.5$\times$ more frequently on Finance (61.3\%) than on General domain (21.7\%), yet detection remains highly effective in both cases. This Model Size Paradox, where larger models fail more but more detectably, suggests that spectral guardrails are most effective precisely where they are most needed.

Our contributions are:
\begin{enumerate}[leftmargin=*,topsep=0pt,itemsep=2pt]
    \item \textbf{A training-free spectral guardrail} that achieves 98.2\% recall on Llama 3.1 8B using a single feature (L26 Smoothness), offering a complementary safety layer to supervised methods without requiring labeled training data.
    \item \textbf{The ``Loud Liar'' discovery:} A controlled cross-model analysis revealing that larger models fail catastrophically and detectably, whereas Mistral achieves superior discrimination (AUC 0.900) with a cleaner geometry of truth.
    \item \textbf{Architecture-specific deployment strategies} derived from a comprehensive evaluation across Qwen, Mistral, and Llama, providing concrete configurations for both safety-critical (recall-optimized) and balanced applications.
\end{enumerate}

\section{Related Work}

\paragraph{Tool Use and Agent Reliability.}
Agent reliability has emerged as a central concern as LLM-based systems are deployed in consequential domains. Tool use benchmarks \citep{patil2024gorilla,qin2023toolllm} evaluate capability but not failure modes. \citet{healy2026internal} demonstrate that supervised probes on hidden states can detect tool-use hallucinations; we address whether comparable detection is possible without labeled training data.

\paragraph{Hallucination Detection.}
White-box methods analyzing internal states have shown promise: INSIDE \citep{chen2024inside}, prompt-guided probes \citep{zhang2025prompt}, deep representation analysis \citep{zhang2025detecting}, HALT \citep{bhatnagar2026halt}, and probabilistic belief propagation \citep{hou2025probabilistic}. Process reward models \citep{lightman2023lets} verify reasoning but require expensive annotation. Our spectral method complements these by operating on attention topology rather than hidden state magnitudes.

\paragraph{Spectral Analysis for LLMs.}
Spectral analysis has theoretical grounding in graph signal processing \citep{shuman2013emerging}. A body of recent work establishes spectral mechanistic interpretability as a framework for understanding transformer computation: \citet{noel2025gsp} introduced attention-induced graphs and spectral diagnostics for hallucination detection and \citet{el2025towards} introduced Attention Graphs for GNN interpretability from a network science perspective. Concurrent work confirms spectral geometry captures LLM computation: \citet{binkowski2025hallucination} use Laplacian eigenvalues; \citet{li2025llm} apply FFT to hidden layer dynamics. We extend this framework to agentic tool-use.

\section{Methods}

\subsection{Threat Model}

We consider an agent deployed in the wild that must call external tools (APIs, databases, system commands). The agent receives user requests, generates tool calls, and executes them. Our guardrail sits between generation and execution, flagging suspicious calls for human review or automatic rejection.

We focus on hallucinations: syntactically valid but semantically incorrect tool calls. These include fabricated function names, incorrect parameter values, and schema violations that would pass basic format checks. Such errors are particularly dangerous because they appear plausible. For safety critical applications, we prioritize recall: it is better to flag a valid call for review than to execute a hallucinated one.

\subsection{Attention as Dynamic Graphs}

Consider a transformer with $L$ layers and $H$ attention heads per layer processing a sequence of $N$ tokens. At layer $\ell \in \{1, \ldots, L\}$ and head $h \in \{1, \ldots, H\}$, let $\mathbf{A}^{(\ell,h)} \in \R^{N \times N}$ denote the post softmax attention matrix, where $A^{(\ell,h)}_{ij}$ represents the attention weight from token $i$ to token $j$. By construction, each row sums to unity: $\sum_{j=1}^N A^{(\ell,h)}_{ij} = 1$.

We interpret each attention matrix as defining a directed weighted graph $\mathcal{G}^{(\ell,h)} = (\mathcal{V}, \mathcal{E}^{(\ell,h)}, \mathbf{A}^{(\ell,h)})$ where vertices $\mathcal{V} = \{1, \ldots, N\}$ correspond to tokens and edge weights are given by attention scores. To enable spectral analysis, we symmetrize to obtain an undirected graph:
\begin{equation}
\mathbf{W}^{(\ell,h)} = \frac{1}{2}\left(\mathbf{A}^{(\ell,h)} + (\mathbf{A}^{(\ell,h)})^\top\right)
\label{eq:symmetrize}
\end{equation}

To obtain a single graph per layer, we aggregate across heads using attention mass weighting:
\begin{equation}
\bar{\mathbf{W}}^{(\ell)} = \sum_{h=1}^{H} \alpha_h^{(\ell)} \mathbf{W}^{(\ell,h)}, \quad \alpha_h^{(\ell)} = \frac{s_h^{(\ell)}}{\sum_{g=1}^{H} s_g^{(\ell)}}
\label{eq:aggregate}
\end{equation}
where $s_h^{(\ell)} = \sum_{i,j} A^{(\ell,h)}_{ij} = N$ is the total attention mass of head $h$.

The combinatorial graph Laplacian of the aggregated attention graph is:
\begin{equation}
\mathbf{L}^{(\ell)} = \bar{\mathbf{D}}^{(\ell)} - \bar{\mathbf{W}}^{(\ell)}
\label{eq:laplacian}
\end{equation}
where $\bar{\mathbf{D}}^{(\ell)} = \diag(\bar{\mathbf{W}}^{(\ell)} \mathbf{1})$ is the degree matrix. The Laplacian is symmetric positive semidefinite with eigendecomposition $\mathbf{L}^{(\ell)} = \mathbf{U}^{(\ell)} \mathbf{\Lambda}^{(\ell)} (\mathbf{U}^{(\ell)})^\top$, where eigenvalues satisfy $0 = \lambda_1 \leq \lambda_2 \leq \cdots \leq \lambda_N$.

\subsection{Spectral Diagnostics}

We extract four spectral diagnostics from each layer.

\begin{definition}[Spectral Entropy]
\label{def:entropy}
The entropy of the spectral energy distribution:
\begin{equation}
\SE^{(\ell)} = -\sum_{m=1}^{N} p_m^{(\ell)} \log p_m^{(\ell)}, \quad p_m^{(\ell)} = \frac{\|\hat{\mathbf{X}}_{m}^{(\ell)}\|_2^2}{\sum_{r} \|\hat{\mathbf{X}}_{r}^{(\ell)}\|_2^2}
\end{equation}
where $\hat{\mathbf{X}}^{(\ell)} = (\mathbf{U}^{(\ell)})^\top \mathbf{X}^{(\ell)}$ is the Graph Fourier Transform of hidden states $\mathbf{X}^{(\ell)} \in \R^{N \times d}$. High entropy indicates diffuse spectral energy across modes; low entropy indicates concentration in few modes. We find that hallucinations exhibit elevated entropy, a thermodynamic signature of incoherent reasoning.
\end{definition}

\begin{definition}[Fiedler Value]
\label{def:fiedler}
The Fiedler value (algebraic connectivity) is the second smallest Laplacian eigenvalue:
\begin{equation}
\lambda_2^{(\ell)} = \min_{\mathbf{x} \perp \mathbf{1}, \|\mathbf{x}\|=1} \mathbf{x}^\top \mathbf{L}^{(\ell)} \mathbf{x}
\end{equation}
This measures how well connected the attention graph is; higher $\lambda_2$ indicates stronger global connectivity and more efficient information flow. Low values indicate fragmented attention patterns.
\end{definition}

\begin{definition}[Smoothness]
\label{def:smoothness}
We define a normalized smoothness measure:
\begin{equation}
\mathcal{S}^{(\ell)} = 1 - \frac{\Tr((\mathbf{X}^{(\ell)})^\top \mathbf{L}^{(\ell)} \mathbf{X}^{(\ell)})}{\lambda_N^{(\ell)} \|\mathbf{X}^{(\ell)}\|_F^2}
\end{equation}
Values near 1 indicate smooth signals where strongly connected tokens have similar representations; values near 0 indicate rough signals with high variation across edges. Hallucinations manifest as Smoothness collapse, particularly dramatic in Llama where L26 Smoothness alone catches 98.2\% of failures.
\end{definition}

\begin{definition}[High Frequency Energy Ratio]
\label{def:hfer}
The HFER measures the proportion of signal energy in high frequency spectral components:
\begin{equation}
\HFER^{(\ell)} = \frac{\sum_{m > N/2} \|\hat{\mathbf{X}}_{m}^{(\ell)}\|_2^2}{\sum_{m=1}^{N} \|\hat{\mathbf{X}}_{m}^{(\ell)}\|_2^2}
\end{equation}
High HFER indicates energy concentration in high frequency, rapidly varying modes, suggesting noisy or incoherent representations.
\end{definition}

\subsection{Classification Protocol}

Given a tool call sequence $\mathcal{T}$, our classification procedure is:
\begin{enumerate}[leftmargin=*,topsep=0pt,itemsep=2pt]
    \item Pass $\mathcal{T}$ through the transformer, extracting attention matrices $\{\mathbf{A}^{(\ell,h)}\}_{\ell,h}$ and hidden states $\{\mathbf{X}^{(\ell)}\}_\ell$.
    \item Compute spectral diagnostics (Entropy, Fiedler, Smoothness, HFER) at each layer.
    \item Apply a threshold rule to classify validity:
\end{enumerate}
\begin{equation}
\hat{y} = \mathbf{1}\left[\mathrm{Metric}^{(\ell^*)} \lessgtr \tau\right]
\end{equation}
where the metric, layer $\ell^*$, direction, and threshold $\tau$ are calibrated on a held out set. We also evaluate multi feature combinations using conjunctive rules.

\subsection{Deployment Considerations}

For agents in the wild, the guardrail must be lightweight. While full eigendecomposition costs $O(N^3)$, we employ the Lanczos approximation \citet{lanczos1950iteration} to efficiently compute only the $k$ spectral components required, reducing complexity to $O(N^2 k)$. For typical tool call lengths ($N < 200$ tokens), this adds approximately 10 to 50ms, acceptable for most applications. For extreme low-latency environments, even faster techniques such as stochastic trace estimation \citep{hutchinson1989stochastic, meyer2021hutch++} or Chebyshev polynomial approximations \citet{hammond2011wavelets} could be utilized to estimate spectral densities in near-linear time. The guardrail can be applied selectively to high-stakes calls (financial transactions) while bypassing low-risk operations.

\section{Experiments}

\subsection{Setup}

We evaluate on the Glaive Function Calling v2 dataset \citet{glaive2024function} with two experimental configurations:

\paragraph{Cross-Model Analysis (Primary).} To enable fair comparison across architectures, we evaluate all three models on the \textbf{same General/Mixed domain} with $N=1000$ samples each at temperature $T=0.3$:
\begin{itemize}[leftmargin=*,topsep=0pt,itemsep=2pt]
    \item Qwen 2.5 0.5B: 200 hallucinations (20.0\% rate)
    \item Mistral 7B v0.1: 207 hallucinations (20.7\% rate)
    \item Llama 3.1 8B: 217 hallucinations (21.7\% rate)
\end{itemize}

This controlled setup isolates architecture-specific failure signatures by matching both domain and hallucination rates across models.

\paragraph{Domain-Specific Analysis (Secondary).} We additionally evaluate:
\begin{itemize}[leftmargin=*,topsep=0pt,itemsep=2pt]
    \item Qwen 2.5 0.5B: Finance ($N=1000$, 174 hallucinations, 17.4\%)
    \item Llama 3.1 8B: Finance ($N=1000$, 613 hallucinations, 61.3\%)
\end{itemize}

All experiments use temperature $T = 0.3$, a common setting for production agents requiring consistency. For reference, we compare against the supervised and baseline approaches from \citet{healy2026internal}, including a 2-layer Multi-Layer Perceptron (MLP) trained on labeled examples, as well as Semantic Similarity \citet{kuhn2023semantic}. This comparison illustrates the fundamental tradeoff between these reference methods, which achieve high precision but miss over 25\% of hallucinations, and our spectral approach, which prioritizes the high recall necessary for safety.

We report results for two optimization objectives: (1) AUC optimization, which maximizes the area under the ROC curve for balanced detection; and (2) Recall optimization, which maximizes the proportion of hallucinations caught, prioritizing safety.

\subsection{Cross-Model Comparison: The Loud Liar Phenomenon}

Our primary contribution is the controlled cross-model analysis on matched domains. Table \ref{tab:cross_model_auc} presents AUC-optimized results.

\begin{table*}[h!]
\caption{Cross-model AUC-optimized configurations on General/Mixed domain ($T=0.3$). All models evaluated on the same domain with matched hallucination rates (20--22\%). Mistral achieves best discrimination (AUC 0.900).}
\label{tab:cross_model_auc}
\begin{center}
\setlength{\tabcolsep}{2pt}
\scriptsize
\begin{tabular}{llcccc}
\toprule
Model & Best Features (5-feat) & AUC & Recall & Precision \\
\midrule
Qwen 2.5 0.5B & L19 Sm.+L23 Ent.+L17 Sm.+L17 HFER+L22 HFER & 0.875 & 82.0\% & 56.7\% \\
Mistral 7B v0.1 & L1 Fied.+L25 Sm.+L5 Sm.+L3 Sm.+L24 Ent. & \textbf{0.900} & 82.1\% & 60.5\% \\
Llama 3.1 8B & L0 Fied.+L24 HFER+L25 HFER+L27 HFER+L30 HFER & 0.845 & 78.8\% & 48.9\% \\
\bottomrule
\end{tabular}
\end{center}
\end{table*}

Mistral 7B achieves the highest AUC (0.900), indicating superior discriminative ability, the boundary between valid and hallucinated calls is sharpest. This suggests Mistral preserves what we might call the ``geometry of truth'' better than other architectures, as shown in Fig. \ref{fig:mistral-combined}.

\begin{figure}[h!]
\centering
\scalebox{0.45}{ 
    \subimport{figures/}{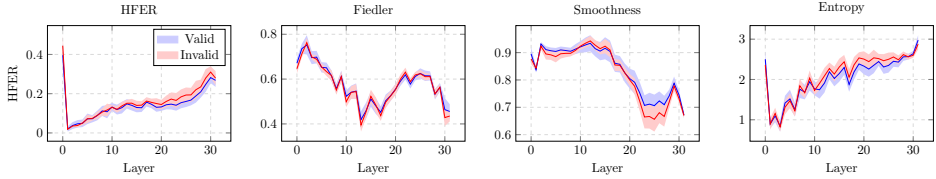}
}
\caption{\textbf{Spectral Metrics Overview (Mistral-7b).} Evolution of HFER, Fiedler value, Smoothness, and Entropy across layers. Valid reasoning (blue) and invalid reasoning (red) show distinct spectral signatures, particularly in terminal layers ($L=20-32$).}
\label{fig:mistral-combined}
\end{figure}

Table \ref{tab:cross_model_recall} presents recall-optimized results, revealing the most striking finding.

\begin{table*}[t]
\caption{Cross-model recall-optimized configurations on General/Mixed domain ($T=0.3$). Llama achieves remarkable 97.7\% recall with 5 features, and 98.2\% with a single feature (L26 Smoothness). This ``Loud Liar'' phenomenon suggests Llama's hallucinations are spectrally catastrophic.}
\label{tab:cross_model_recall}
\begin{center}
\setlength{\tabcolsep}{2pt}
\scriptsize
\begin{tabular}{llcccc}
\toprule
Model & Configuration & Recall & Detected & Precision & AUC \\
\midrule
Qwen 2.5 0.5B & L19 Sm.+L17 Sm.+L17 HFER (3-feat) & 86.5\% & 173/200 & 45.1\% & 0.847 \\
Mistral 7B v0.1 & L1 Fied.+L5 Sm.+L10 Fied.+L6 Fied.+L25 HFER & 91.3\% & 189/207 & 38.3\% & 0.847 \\
Llama 3.1 8B & L0 Fied.+L1 Fied.+L20 HFER+L22 HFER+L29 HFER & \textbf{97.7\%} & 212/217 & 34.0\% & 0.789 \\
\midrule
\multicolumn{6}{l}{\textit{Single-Feature Detection (Maximum Recall)}} \\
\midrule
Llama 3.1 8B & L26 Smoothness (single) & \textbf{98.2\%} & 213/217 & 23.6\% & 0.541 \\
Mistral 7B v0.1 & L3 Entropy (single) & 94.7\% & 196/207 & 22.0\% & --- \\
Qwen 2.5 0.5B & L7 Smoothness (single) & 78.0\% & 156/200 & 40.0\% & --- \\
\bottomrule
\end{tabular}
\end{center}
\end{table*}

\paragraph{The Loud Liar Phenomenon.} Llama 3.1 8B exhibits what we term the ``Loud Liar'' phenomenon: its hallucinations are spectrally catastrophic. A single feature, L26 Smoothness, catches \textbf{98.2\%} of all hallucinations (213 of 217). This is remarkable: when Llama hallucinates, its internal representation ``collapses'' dramatically in spectral space. It does not lie smoothly; it breaks spectrally.

This has profound implications for deployment: Llama may produce more hallucinations in challenging domains, but its failures are highly policeable. The 5-feature configuration achieves 97.7\% recall, missing only 5 hallucinations out of 217.

\paragraph{Mistral: The Clean Discriminator.} In contrast, Mistral achieves the best AUC (0.900) but lower maximum recall (91.3\%). This suggests Mistral maintains a cleaner separation between valid and invalid calls, its hallucinations overlap less with the valid distribution in spectral space. For applications prioritizing precision over exhaustive recall, Mistral offers the best tradeoff.

\paragraph{Qwen: Balanced but Limited.} Qwen achieves moderate recall (86.5\%) with the best single-feature precision (40.0\%). Its smaller size may constrain both the severity and distinctiveness of its failures.

\subsection{Domain-Specific Analysis: The Model Size Paradox}

Beyond the controlled cross-model comparison, we examine domain effects by evaluating Llama on Finance versus General domains.

\begin{table}[h]
\caption{Hallucination rates across models and domains. Larger models are not necessarily safer; domain difficulty varies substantially.}
\label{tab:hallucination_rates}
\begin{center}
\begin{tabular}{llcc}
\toprule
Model & Domain & Hallucinations & Rate \\
\midrule
Qwen 2.5 0.5B & Finance & 174/1000 & 17.4\% \\
Qwen 2.5 0.5B & General/Mixed & 200/1000 & 20.0\% \\
Mistral 7B v0.1 & General & 207/1000 & 20.7\% \\
Llama 3.1 8B & General & 217/1000 & 21.7\% \\
Llama 3.1 8B & Finance & 613/1000 & 61.3\% \\
\bottomrule
\end{tabular}
\end{center}
\end{table}

Llama 3.1 8B hallucinates 2.8$\times$ more frequently on Finance (61.3\%) than General (21.7\%), revealing substantial domain difficulty effects. On the controlled General domain comparison, all three models show similar hallucination rates (20--22\%), isolating architecture as the key variable. The Finance domain appears particularly challenging for Llama, possibly due to specialized terminology or schema complexity.

For practitioners deploying agents in the wild, this finding is critical: domain choice may matter more than model size for raw reliability, but spectral detection remains effective regardless.

\subsection{AUC Optimized Results}

Table \ref{tab:auc_results} presents AUC optimized configurations for domain-specific evaluation.

\begin{table*}[t]
\caption{AUC optimized configurations across architectures and domains. All results show strong discrimination with Cohen's $d > 0.5$ and $p < 10^{-16}$.}
\label{tab:auc_results}
\scriptsize
\begin{center}
\setlength{\tabcolsep}{2pt}
\begin{tabular}{lllccccc}
\toprule
Model & Domain & Best Features & AUC [95\% CI] & Recall & $d$ & $p$ \\
\midrule
Qwen 0.5B & Finance & L19+L21+L22 Smooth. (triplet) & 0.844 [0.809, 0.877] & 81.0\% & 1.190 & $<10^{-39}$ \\
Qwen 0.5B & General & L20+L22 Smooth. (pair) & 0.826 [0.791, 0.858] & 73.5\% & 0.944 & $<10^{-34}$ \\
Mistral 7B & General & L1 Fied.+L3,5,25 Sm.+L24 Ent. & 0.900 [0.874, 0.925] & 82.1\% & 1.124 & $<10^{-40}$ \\
Llama 8B & Finance & L15 Fied.+L26 HFER (pair) & 0.817 [0.790, 0.844] & 78.6\% & 0.541 & $<10^{-18}$ \\
\bottomrule
\end{tabular}
\end{center}
\end{table*}

Mistral 7B achieves the best overall AUC (0.900) using a combination of Fiedler, Smoothness, and Entropy features. Qwen achieves strong performance on both Finance (AUC 0.844) and General/Mixed domains (AUC 0.826). Effect sizes are large across all models ($d = 0.541$ to $1.449$), and all results are highly significant ($p < 10^{-16}$).

\subsection{Recall Optimized Results: The Nuclear Option}

For safety critical applications, we optimize for maximum recall. Table \ref{tab:recall_results} presents our recall optimized configurations for domain-specific evaluation. See Appendix \ref{app:confidence_failure} for an analysis of why standard confidence metrics fail.

\begin{table}[h!]
\caption{Recall optimized configurations. Spectral methods achieve superior recall (86.1\%) compared to all baselines. Semantic Similarity achieve high precision but fail to detect over a quarter of hallucinations (Recall $\approx$ 73\%).}
\label{tab:recall_results}
\begin{center}
\scriptsize
\setlength{\tabcolsep}{2pt}
\begin{tabular}{llccc}
\toprule
Model & Configuration & Recall & Detected & Precision \\
\midrule
\multicolumn{5}{l}{\textit{Qwen 2.5 0.5B General/Mixed (200 hallucinations, 20.0\% rate)}} \\
\midrule
& L7 Smoothness (single) & 78.0\% & 156/200 & 40.0\% \\
& L19 Sm.+L17 Sm.+L17 HFER (triplet) & 86.5\% & 173/200 & 45.1\% \\
& L20 Sm.+L23 Sm.+L21 Ent.+L18 HFER+L17 HFER & 86.5\% & 173/200 & 40.7\% \\
\midrule
\multicolumn{5}{l}{\textit{Mistral 7B v0.1 General (207 hallucinations, 20.7\% rate)}} \\
\midrule
& L3 Entropy (single) & 94.7\% & 196/207 & 22.0\% \\
& L1 Fied.+L5 Sm.+L10 Fied.+L25 HFER (4 feat.) & 90.8\% & 188/207 & 38.3\% \\
& L1 Fied.+L5 Sm.+L10 Fied.+L6 Fied.+L25 HFER & 91.3\% & 189/207 & 38.3\% \\
\midrule
\multicolumn{5}{l}{\textit{Llama 3.1 8B General (217 hallucinations, 21.7\% rate)}} \\
\midrule
& L26 Smoothness (single) & \textbf{98.2\%} & 213/217 & 23.6\% \\
& L0 Fied.+L1 Fied. (pair) & 94.5\% & 205/217 & 35.4\% \\
& L0 Fied.+L21 Ent.+L1 Fied. (triplet) & 95.4\% & 207/217 & 35.1\% \\
\midrule
\multicolumn{5}{l}{\textit{Llama 3.1 8B Finance (613 hallucinations, 61.3\% rate)}} \\
\midrule
& L27 Entropy (single) & 95.8\% & 587/613 & 64.4\% \\
& L28 Fiedler+L30 Fiedler (pair) & 89.7\% & 550/613 & 69.4\% \\
& L30 Fied.+L1 Ent.+L25 HFER (triplet) & \textbf{86.1\%} & 528/613 & 81.0\% \\
& L29 Fied.+L30 Sm.+L28 Fied.+L30 Fied.+L19 HFER & 88.7\% & 544/613 & 79.6\% \\
\midrule
\multicolumn{5}{l}{\textit{Reference Baselines \citet{healy2026internal}}} \\
\midrule
& Supervised MLP (Hidden States) & 53.0 -- 62.0\% & --- & 71.0 -- 86.0\%  \\
& Semantic Similarity & 73.1\% & --- & \textbf{100.0\%} \\
\bottomrule
\end{tabular}
\end{center}
\end{table}

\subsection{Comparison to Supervised Approaches}

On Llama 3.1 8B Finance, our triplet configuration (L30 Fiedler + L1 Entropy + L25 HFER) achieves 86.1\% recall with 81.0\% precision using zero training data. This demonstrates that spectral features capture meaningful structure for hallucination detection.

The key advantage of spectral methods is immediate deployability: when a new model or tool schema is introduced, spectral guardrails can be applied with only threshold calibration on a small held out set (50--100 examples), rather than collecting thousands of labeled training examples. Supervised probes and spectral guardrails offer complementary strengths, and combining both approaches may yield further improvements.

\subsection{The Spectral Signature Discovery}

Our most significant finding is that single-layer spectral features achieve near-perfect hallucination detection:
\begin{itemize}[leftmargin=*,topsep=0pt,itemsep=2pt]
    \item \textbf{Llama L26 Smoothness} (General): 98.2\% recall (213/217 caught), 23.6\% precision
    \item \textbf{Llama L27 Entropy} (Finance): 95.8\% recall (587/613 caught), 64.4\% precision
    \item \textbf{Mistral L3 Entropy}: 94.7\% recall (196/207 caught), 22.0\% precision
    \item \textbf{Qwen L7 Smoothness}: 78.0\% recall (156/200 caught), 40.0\% precision
\end{itemize}

This suggests a deeper interpretation: hallucination is a thermodynamic state change. When a model generates valid output, its spectral energy concentrates in coherent, low entropy patterns corresponding to structured reasoning. When it hallucinates, energy disperses across modes, the attention mechanism becomes noise. Entropy and Smoothness capture this phase transition with a single scalar.

The fact that a simple threshold on one layer achieves 98\%+ recall, with no training, is remarkable. It suggests that hallucination is not a subtle phenomenon but a catastrophic failure mode that manifests clearly in the attention spectrum.

\paragraph{Perfect Precision Mode.}
For applications requiring zero false positives (e.g., auto-rejection without human review), 5-feature combinations achieve 100\% precision at the cost of low recall (10.1\%). The configuration L25 HFER + L1 Fiedler + L23 HFER + L27 HFER + L30 HFER catches only 22/217 hallucinations but with zero false alarms (AUC 0.819, $d=1.39$). This enables tiered guardrails: auto-reject the obvious failures, flag the rest for review.  L1 Fiedler appears in ALL top-5 configs. This suggests early-layer connectivity + late-layer HFER isolates the "most catastrophic" hallucinations,the ones so broken they're unmistakable.

\subsection{Surgical versus Diffuse Detection}

The precision gap between models reveals different failure geometries:
\begin{itemize}[leftmargin=*,topsep=0pt,itemsep=2pt]
    \item \textbf{Llama} (64.4\% precision at 95.8\% recall on Finance; 23.6\% at 98.2\% on General): The model's hallucinations occupy a distinct spectral region, well separated from valid outputs. Higher precision on Finance reflects the higher base rate.
    \item \textbf{Mistral} (22.0\% precision at 94.7\% recall): Diffuse detection. The detector catches nearly every error but also flags many valid calls, suggesting overlap between the spectral distributions.
    \item \textbf{Qwen} (40.0\% precision at 78.0\% recall): Moderate overlap between valid and hallucinated distributions.
\end{itemize}

This difference reflects the Loud Liar phenomenon: larger models have distinct, confident failure modes that are geometrically isolated. When Llama hallucinates, it does so decisively. Smaller models fail more diffusely, their hallucinations blend into the valid distribution.

\subsection{Architecture Specific Failure Modes}

Table \ref{tab:signatures} summarizes the architecture specific detection patterns from the controlled cross-model analysis.

\begin{table}[h]
\caption{Architecture specific failure modes and optimal monitoring strategies.}
\label{tab:signatures}
\begin{center}
\footnotesize
\begin{tabular}{lccc}
\toprule
Property & Qwen 0.5B & Mistral 7B & Llama 8B \\
\midrule
Hallucination Rate & 20.0\% & 20.7\% & 21.7\% \\
Best AUC & 0.875 & \textbf{0.900} & 0.845 \\
Max Recall (single) & 78.0\% & 94.7\% & \textbf{98.2\%} \\
Max Recall (multi) & 86.5\% & 91.3\% & \textbf{97.7\%} \\
Single-feat Precision & 40.0\% & 22.0\% & 23.6\% \\
Best Single Feature & L7 Smooth. & L3 Entropy & L26 Smooth. \\
Dominant Metric & Smoothness & Mixed & HFER+Smooth. \\
Key Layers & L17--L23 & L1--L10, L25 & L20--L30 \\
Layer Position & Late & Early + Late & Final \\
\bottomrule
\end{tabular}
\end{center}
\end{table}

\textbf{Qwen} relies on mid to late layer Smoothness features (L17 to L23). Failures manifest as Smoothness collapse in the final processing stages, when the model must commit to output tokens.

\textbf{Mistral} uses early layer Entropy (L3) for maximum recall but requires distributed Fiedler features (L1, L6, L10) for precision. The sliding window attention architecture creates layerwise heterogeneous failure patterns.

\textbf{Llama} concentrates detection in final layers (L20 to L30), with HFER and Smoothness providing the strongest signals. The model's failures are geometrically distinct and occur at the final stages of generation, the ``Loud Liar'' signature.

\section{Deployment Recommendations}

Based on our comprehensive analysis, we offer concrete guidance for practitioners deploying agents in the wild.

\paragraph{Maximum Safety (Recall Priority).}
For applications where missed hallucinations are unacceptable, use single spectral features:
\begin{itemize}[leftmargin=*,topsep=0pt,itemsep=2pt]
    \item \textbf{Llama 3.1 8B}: Use L26 Smoothness alone. \textbf{98.2\% recall}, 23.6\% precision. Catches 213/217 hallucinations. Or use 5-feature config for 97.7\% recall with 34.0\% precision.
    \item \textbf{Mistral 7B}: Use L3 Entropy alone. 94.7\% recall, 22.0\% precision. Catches 196/207 hallucinations.
    \item \textbf{Qwen 2.5 0.5B}: Use L19+L17+L17 (Smooth.+Smooth.+HFER). 86.5\% recall, 45.1\% precision.
\end{itemize}

\paragraph{Balanced Deployment (Recall + Precision).}
For applications requiring both safety and usability:
\begin{itemize}[leftmargin=*,topsep=0pt,itemsep=2pt]
    \item \textbf{Llama 3.1 8B} (Finance): Use L30 Fiedler + L1 Entropy + L25 HFER. 86.1\% recall, \textbf{81.0\% precision}. This is the best overall configuration for balanced deployment.
    \item \textbf{Mistral 7B}: Use 5-feature config. 91.3\% recall, 38.3\% precision.
    \item \textbf{Qwen 2.5 0.5B}: Use triplet Smoothness. 86.5\% recall, 45.1\% precision.
\end{itemize}

\paragraph{Best Discrimination (AUC Priority).}
For research or applications requiring balanced discrimination:
\begin{itemize}[leftmargin=*,topsep=0pt,itemsep=2pt]
    \item \textbf{Mistral 7B} achieves best AUC (\textbf{0.900}) with multi feature configuration, cleanest separation between valid and invalid.
    \item Qwen 2.5 0.5B achieves AUC 0.875 on General, 0.844 on Finance.
    \item Llama 3.1 8B achieves AUC 0.845 on General, 0.817 on Finance.
\end{itemize}

\paragraph{Model Selection Strategy.}
The choice depends on operational priorities:
\begin{enumerate}[leftmargin=*,topsep=0pt,itemsep=2pt]
    \item \textbf{If recall is paramount}: Deploy Llama with L26 Smoothness monitoring. Accept higher false positive rate in exchange for catching 98\%+ of failures. The ``Loud Liar'' phenomenon makes Llama's errors highly policeable.
    \item \textbf{If discrimination matters}: Deploy Mistral with multi-feature monitoring. Best AUC (0.900) indicates cleanest geometry of truth.
    \item \textbf{If resources are constrained}: Deploy Qwen with Smoothness monitoring. Smaller model with reasonable detection (86.5\% recall).
\end{enumerate}

\paragraph{Threshold Calibration.}
Calibrate thresholds on 50 to 100 labeled calls per deployment context. Thresholds do not transfer across models or substantially different tool schemas. Recalibration is required when changing models or domains.

\section{Limitations and Future Work}

Our evaluation covers three model families at $T=0.3$ with temperature ablations up to $T=1.5$. We find that detection capability degrades with increasing temperature (see Appendix), as higher temperature adds entropic noise that masks spectral signatures. Generalization to other architectures and tool schemas requires further study. The precision of single feature detection (22--24\% on General domain) may be insufficient for some applications, necessitating multi feature configurations with associated complexity.


Future work should explore hybrid guardrails combining spectral features with supervised probes, extension to multi-turn agentic interactions and planning where context accumulates across steps, and theoretical analysis of why larger models produce 'louder' spectral failures.

\section{Conclusion}

We present a training-free spectral guardrail for tool use hallucination detection. On Llama 3.1 8B, our method achieves 98.2\% recall with a single feature and 86.1\% recall (81.0\% precision) in balanced settings, demonstrating that spectral features effectively complement supervised approaches.

The discovery that single-layer features achieve such high recall suggests hallucination is a thermodynamic phase transition: a measurable collapse from coherent to entropic attention. This insight provides a novel theoretical lens for understanding model failures as structural energy dispersion.

Our cross-model analysis reveals the ``Loud Liar'' phenomenon: Llama's hallucinations are spectrally catastrophic and easily detected, while Mistral achieves the best discrimination (AUC 0.900). These architecture-specific failure geometries dictate distinct deployment strategies.

These findings establish spectral analysis as a principled framework for agent safety. With near-perfect recall and zero training requirements, our approach is particularly valuable for practitioners where labeled data is scarce or deployment contexts evolve rapidly. Future work may explore hybrid spectral-supervised guardrails for even greater reliability.


\newpage

\begin{thebibliography}{23}
\providecommand{\natexlab}[1]{#1}
\providecommand{\url}[1]{\texttt{#1}}
\expandafter\ifx\csname urlstyle\endcsname\relax
  \providecommand{\doi}[1]{doi: #1}\else
  \providecommand{\doi}{doi: \begingroup \urlstyle{rm}\Url}\fi

\bibitem[Azaria and Mitchell(2023)]{azaria2023internal}
Amos Azaria and Tom Mitchell.
\newblock The internal state of an {LLM} knows when it's lying.
\newblock \emph{arXiv preprint arXiv:2304.13734}, 2023.

\bibitem[Bhatnagar et~al.(2026)Bhatnagar, Sun, Zhang, Wen, and Yang]{bhatnagar2026halt}
Rohan Bhatnagar, Youran Sun, Chi~Andrew Zhang, Yixin Wen, and Haizhao Yang.
\newblock Halt: Hallucination assessment via latent testing.
\newblock \emph{arXiv preprint arXiv:2601.14210}, 2026.

\bibitem[Binkowski et~al.(2025)Binkowski, Janiak, Sawczyn, Gabrys, and Kajdanowicz]{binkowski2025hallucination}
Jakub Binkowski, Denis Janiak, Albert Sawczyn, Bogdan Gabrys, and Tomasz~Jan Kajdanowicz.
\newblock Hallucination detection in llms using spectral features of attention maps.
\newblock In \emph{Proceedings of the 2025 Conference on Empirical Methods in Natural Language Processing}, pages 24365--24396, 2025.

\bibitem[Chen et~al.(2024)Chen, Liu, Chen, Gu, Wu, Tao, Fu, and Ye]{chen2024inside}
Chao Chen, Kai Liu, Ze~Chen, Yi~Gu, Yue Wu, Mingyuan Tao, Zhihang Fu, and Jieping Ye.
\newblock Inside: Llms' internal states retain the power of hallucination detection.
\newblock \emph{ArXiv}, abs/2402.03744, 2024.
\newblock URL \url{https://api.semanticscholar.org/CorpusID:267499843}.

\bibitem[Chung(1997)]{chung1997spectral}
Fan R.~K. Chung.
\newblock \emph{Spectral Graph Theory}, volume~92 of \emph{CBMS Regional Conference Series in Mathematics}.
\newblock American Mathematical Society, 1997.

\bibitem[El et~al.(2025)El, Choudhury, Li{\`o}, and Joshi]{el2025towards}
Batu El, Deepro Choudhury, Pietro Li{\`o}, and Chaitanya~K Joshi.
\newblock Towards mechanistic interpretability of graph transformers via attention graphs.
\newblock \emph{arXiv preprint arXiv:2502.12352}, 2025.

\bibitem[{GlaiveAI}(2024)]{glaive2024function}
{GlaiveAI}.
\newblock Glaive function calling v2.
\newblock \url{https://huggingface.co/datasets/glaiveai/glaive-function-calling-v2}, 2024.
\newblock Accessed: 2025-01-27.

\bibitem[Hammond et~al.(2011)Hammond, Vandergheynst, and Gribonval]{hammond2011wavelets}
David~K. Hammond, Pierre Vandergheynst, and R{\'e}mi Gribonval.
\newblock Wavelets on graphs via spectral graph theory.
\newblock In \emph{Applied and Computational Harmonic Analysis}, volume~30, pages 129--150, 2011.

\bibitem[Healy et~al.(2026)Healy, Srinivasan, Madathil, and Wu]{healy2026internal}
Kait Healy, Bharathi Srinivasan, Visakh Madathil, and Jing Wu.
\newblock Internal representations as indicators of hallucinations in agent tool selection.
\newblock \emph{arXiv preprint arXiv:2601.05214}, 2026.

\bibitem[Hou et~al.(2025)Hou, Zhang, Andreas, and Chang]{hou2025probabilistic}
Bairu Hou, Yang Zhang, Jacob Andreas, and Shiyu Chang.
\newblock A probabilistic framework for llm hallucination detection via belief tree propagation.
\newblock In \emph{Proceedings of the 2025 Conference of the Nations of the Americas Chapter of the Association for Computational Linguistics: Human Language Technologies (Volume 1: Long Papers)}, pages 3076--3099, 2025.

\bibitem[Hutchinson(1989)]{hutchinson1989stochastic}
Michael~F Hutchinson.
\newblock A stochastic estimator of the trace of the influence matrix for laplacian smoothing splines.
\newblock \emph{Communications in Statistics-Simulation and Computation}, 18\penalty0 (3):\penalty0 1059--1076, 1989.

\bibitem[Jiang et~al.(2023)Jiang, Sablayrolles, Mensch, Bamford, Chaplot, Casas, Bressand, Lengyel, Lample, Saulnier, et~al.]{jiang2023mistral}
Albert~Q. Jiang, Alexandre Sablayrolles, Arthur Mensch, Chris Bamford, Devendra~Singh Chaplot, Diego de~las Casas, Florian Bressand, Gianna Lengyel, Guillaume Lample, Lucile Saulnier, et~al.
\newblock Mistral 7{B}.
\newblock \emph{arXiv preprint arXiv:2310.06825}, 2023.

\bibitem[Kuhn et~al.(2023)Kuhn, Gal, and Farquhar]{kuhn2023semantic}
Lorenz Kuhn, Yarin Gal, and Sebastian Farquhar.
\newblock Semantic uncertainty: Linguistic invariances for uncertainty estimation in natural language generation.
\newblock \emph{arXiv preprint arXiv:2302.09664}, 2023.

\bibitem[Lanczos(1950)]{lanczos1950iteration}
Cornelius Lanczos.
\newblock An iteration method for the solution of the eigenvalue problem of linear differential and integral operators.
\newblock \emph{Journal of research of the National Bureau of Standards}, 45\penalty0 (4):\penalty0 255--282, 1950.

\bibitem[Li et~al.(2025)Li, Tu, and Hu]{li2025llm}
JinXin Li, Gang Tu, and JunJie Hu.
\newblock Llm hallucination detection: Hsad.
\newblock \emph{arXiv preprint arXiv:2509.23580}, 2025.

\bibitem[Lightman et~al.(2023)Lightman, Kosaraju, Burda, Edwards, Baker, Lee, Leike, Schulman, Sutskever, and Cobbe]{lightman2023lets}
Hunter Lightman, Vineet Kosaraju, Yura Burda, Harri Edwards, Bowen Baker, Teddy Lee, Jan Leike, John Schulman, Ilya Sutskever, and Karl Cobbe.
\newblock Let's verify step by step.
\newblock \emph{arXiv preprint arXiv:2305.20050}, 2023.

\bibitem[Meyer et~al.(2021)Meyer, Musco, Musco, and Woodruff]{meyer2021hutch++}
Raphael~A Meyer, Cameron Musco, Christopher Musco, and David~P Woodruff.
\newblock Hutch++: Optimal stochastic trace estimation.
\newblock In \emph{Symposium on Simplicity in Algorithms (SOSA)}, pages 142--155. SIAM, 2021.

\bibitem[Noël(2025)]{noel2025gsp}
Valentin Noël.
\newblock A graph signal processing framework for hallucination detection in large language models, 2025.
\newblock URL \url{https://arxiv.org/abs/2510.19117}.

\bibitem[Patil et~al.(2024)Patil, Zhang, Wang, and Gonzalez]{patil2024gorilla}
Shishir~G Patil, Tianjun Zhang, Xin Wang, and Joseph~E Gonzalez.
\newblock Gorilla: Large language model connected with massive apis.
\newblock \emph{Advances in Neural Information Processing Systems}, 37:\penalty0 126544--126565, 2024.

\bibitem[Qin et~al.(2023)Qin, Liang, Ye, Zhu, Yan, Lu, Lin, Cong, Tang, Qian, et~al.]{qin2023toolllm}
Yujia Qin, Shihao Liang, Yining Ye, Kunlun Zhu, Lan Yan, Yaxi Lu, Yankai Lin, Xin Cong, Xiangru Tang, Bill Qian, et~al.
\newblock Toolllm: Facilitating large language models to master 16000+ real-world apis.
\newblock \emph{arXiv preprint arXiv:2307.16789}, 2023.

\bibitem[Shuman et~al.(2013)Shuman, Narang, Frossard, Ortega, and Vandergheynst]{shuman2013emerging}
David~I. Shuman, Sunil~K. Narang, Pascal Frossard, Antonio Ortega, and Pierre Vandergheynst.
\newblock The emerging field of signal processing on graphs: Extending high-dimensional data analysis to networks and other irregular domains.
\newblock \emph{IEEE Signal Processing Magazine}, 30\penalty0 (3):\penalty0 83--98, 2013.

\bibitem[Zhang et~al.(2025{\natexlab{a}})Zhang, Yu, Yi, Zhang, Li, and Liu]{zhang2025prompt}
Fujie Zhang, Peiqi Yu, Biao Yi, Baolei Zhang, Tong Li, and Zheli Liu.
\newblock Prompt-guided internal states for hallucination detection of large language models.
\newblock In \emph{Proceedings of the 63rd Annual Meeting of the Association for Computational Linguistics (Volume 1: Long Papers)}, pages 21806--21818, 2025{\natexlab{a}}.

\bibitem[Zhang et~al.(2025{\natexlab{b}})Zhang, Song, Wu, Tian, Zhou, Xu, Yang, and Zhang]{zhang2025detecting}
Luan Zhang, Dandan Song, Zhijing Wu, Yuhang Tian, Changzhi Zhou, Jing Xu, Ziyi Yang, and Shuhao Zhang.
\newblock Detecting hallucination in large language models through deep internal representation analysis.
\newblock In \emph{Proceedings of the Thirty-Fourth International Joint Conference on Artificial Intelligence, IJCAI-25}, pages 8357--8365, 2025{\natexlab{b}}.

\end{thebibliography}

\newpage
\appendix

\section{Complete Cross-Model Results (General/Mixed Domain)}

\subsection{AUC-Optimized Configurations}

\begin{table}[h!]
\caption{Complete AUC-optimized results for all models on General/Mixed domain ($T=0.3$, controlled comparison).}
\begin{center}
\setlength{\tabcolsep}{2pt}
\scriptsize
\begin{tabular}{llccc}
\toprule
Model & Features (5-feat) & AUC & Recall & Precision \\
\midrule
Qwen 2.5 0.5B & L19 Sm.+L23 Ent.+L17 Sm.+L17 HFER+L22 HFER & 0.875 & 82.0\% & 56.7\% \\
Mistral 7B v0.1 & L1 Fied.+L25 Sm.+L5 Sm.+L3 Sm.+L24 Ent. & 0.900 & 82.1\% & 60.5\% \\
Llama 3.1 8B & L0 Fied.+L24 HFER+L25 HFER+L27 HFER+L30 HFER & 0.845 & 78.8\% & 48.9\% \\
\bottomrule
\end{tabular}
\end{center}
\end{table}

\subsection{Recall-Optimized Configurations}

\begin{table}[h]
\caption{Complete recall-optimized results for all models on General/Mixed domain ($T=0.3$, controlled comparison).}
\begin{center}
\setlength{\tabcolsep}{2pt}
\scriptsize
\begin{tabular}{llcccc}
\toprule
Model & Features & Recall & Detected & Precision & AUC \\
\midrule
Qwen 2.5 0.5B & L19 Sm.+L17 Sm.+L17 HFER (3-feat) & 86.5\% & 173/200 & 45.1\% & 0.847 \\
Mistral 7B v0.1 & L1 Fied.+L5 Sm.+L10 Fied.+L6 Fied.+L25 HFER & 91.3\% & 189/207 & 38.3\% & 0.847 \\
Llama 3.1 8B & L0 Fied.+L1 Fied.+L20 HFER+(L22,L29) HFER & 97.7\% & 212/217 & 34.0\% & 0.789 \\
\midrule
\multicolumn{6}{l}{\textit{Single-Feature Detection}} \\
\midrule
Llama 3.1 8B & L26 Smoothness & 98.2\% & 213/217 & 23.6\% & 0.541 \\
Mistral 7B v0.1 & L3 Entropy & 94.7\% & 196/207 & 22.0\% & --- \\
Qwen 2.5 0.5B & L7 Smoothness & 78.0\% & 156/200 & 40.0\% & --- \\
\bottomrule
\end{tabular}
\end{center}
\end{table}

\newpage

\section{Complete Domain-Specific Results}

\subsection{Qwen 2.5 0.5B: Finance Domain}

\begin{table}[h]
\caption{All AUC optimized configurations for Qwen 2.5 0.5B on Finance ($N=1000$, 174 hallucinations, 17.4\% rate). All features are Smoothness.}
\begin{center}
\setlength{\tabcolsep}{2pt}
\scriptsize
\begin{tabular}{llcccccc}
\toprule
Config & Features & AUC [95\% CI] & Recall & Detected & Prec. & $d$ & $p$ \\
\midrule
Single & L19 Sm. & 0.816 [0.778, 0.852] & 73.6\% & 128/174 & 53.5\% & 1.449 & $<10^{-39}$ \\
Pair & L20 Sm.+L22 Sm. & 0.836 [0.801, 0.868] & 63.2\% & 110/174 & 54.2\% & 1.006 & $<10^{-29}$ \\
Triplet & L19 Sm.+L21 Sm.+L22 Sm. & 0.844 [0.809, 0.877] & 81.0\% & 141/174 & 41.0\% & 1.190 & $<10^{-39}$ \\
Quad & L20 Sm.+L22 Sm.+L23 Sm.+L22 Ent. & 0.752 [0.706, 0.799] & 52.9\% & 92/174 & 49.7\% & 1.096 & $<10^{-29}$ \\
\bottomrule
\end{tabular}
\end{center}
\end{table}

\subsection{Qwen 2.5 0.5B: General/Mixed Domain}

\begin{table}[h]
\caption{All AUC optimized configurations for Qwen 2.5 0.5B on General/Mixed Domain ($N=1000$, 200 hallucinations, 20.0\% rate). Smoothness dominates.}
\begin{center}
\setlength{\tabcolsep}{2pt}
\scriptsize
\begin{tabular}{llcccccc}
\toprule
Config & Features & AUC [95\% CI] & Recall & Detected & Prec. & $d$ & $p$ \\
\midrule
Single & L19 Sm. & 0.812 [0.779, 0.846] & 73.0\% & 146/200 & --- & 1.448 & $<10^{-42}$ \\
Pair & L20 Sm.+L22 Sm. & 0.826 [0.791, 0.858] & 73.5\% & 147/200 & 48.2\% & 0.944 & $<10^{-34}$ \\
Triplet & L19 Sm.+L17 Sm.+L17 HFER & 0.764 [0.725, 0.801] & 75.0\% & 150/200 & 40.2\% & 0.975 & $<10^{-42}$ \\
Quad & L19 Sm.+L23 Ent.+L17 Sm.+L17 HFER & 0.750 [0.707, 0.791] & 55.5\% & 111/200 & 50.7\% & 0.987 & $<10^{-42}$ \\
\bottomrule
\end{tabular}
\end{center}
\end{table}

\newpage

\subsection{Llama 3.1 8B: Finance Domain}

\begin{table}[h]
\caption{All AUC optimized configurations for Llama 3.1 8B on Finance ($N=1000$, 613 hallucinations, 61.3\% rate). Fiedler and HFER dominate.}
\begin{center}
\setlength{\tabcolsep}{2pt}
\scriptsize
\begin{tabular}{llcccccc}
\toprule
Config & Features & AUC [95\% CI] & Recall & Detected & Prec. & $d$ & $p$ \\
\midrule
Single & L29 Fied. & 0.696 [0.664, 0.726] & 75.0\% & 460/613 & 79.6\% & 0.617 & $<10^{-16}$ \\
Pair & L15 Fied.+L26 HFER & 0.817 [0.790, 0.844] & 78.6\% & 482/613 & 83.0\% & 0.541 & $<10^{-18}$ \\
Triplet & L1 Ent.+L15 Fied.+L26 HFER & 0.748 [0.715, 0.777] & 77.0\% & 472/613 & 76.9\% & 0.549 & $<10^{-19}$ \\
Quad & L30 Fied.+L1 Ent.+L15 Fied.+L12 HFER & 0.737 [0.703, 0.767] & 73.1\% & 448/613 & 77.9\% & 0.562 & $<10^{-23}$ \\
\bottomrule
\end{tabular}
\end{center}
\end{table}

\subsection{Mistral 7B: General Domain}

\begin{table}[h]
\caption{Best configuration for Mistral 7B on General Domain ($N=1000$, 207 hallucinations, 20.7\% rate).}
\begin{center}
\scriptsize
\setlength{\tabcolsep}{2pt}
\begin{tabular}{llccccc}
\toprule
Config & Features & AUC [95\% CI] & Recall & Detected & Prec. & $d$ \\
\midrule
Quint & L1 Fied.+L25 Sm.+L5 Sm.+L3 Sm.+L24 Ent. & 0.900 [0.874, 0.925] & 82.1\% & 170/207 & 60.5\% & 1.124 \\
\bottomrule
\end{tabular}
\end{center}
\end{table}

\newpage

\section{Complete Recall Optimized Results (Domain-Specific)}

\subsection{Qwen 2.5 0.5B}

\begin{table}[h]
\caption{Qwen 2.5 0.5B recall optimized on General/Mixed ($N=1000$, 200 hallucinations, 20.0\% rate).}
\begin{center}
\footnotesize
\begin{tabular}{lccc}
\toprule
Features & Recall & Detected & Precision \\
\midrule
L7 Smoothness & 78.0\% & 156/200 & 40.0\% \\
L19 Sm.+L17 Sm.+L17 HFER & 86.5\% & 173/200 & 45.1\% \\
L20 Sm.+L23 Sm.+L21 Ent.+L18 HFER+L17 HFER & 86.5\% & 173/200 & 40.7\% \\
\bottomrule
\end{tabular}
\end{center}
\end{table}

\subsection{Mistral 7B v0.1}

\begin{table}[h]
\caption{Mistral 7B recall optimized on General ($N=1000$, 207 hallucinations, 20.7\% rate).}
\begin{center}
\footnotesize
\begin{tabular}{lccc}
\toprule
Features & Recall & Detected & Precision \\
\midrule
L3 Entropy & 94.7\% & 196/207 & 22.0\% \\
L1 Fied.+L5 Sm.+L10 Fied.+L25 HFER & 90.8\% & 188/207 & 38.3\% \\
L1 Fied.+L5 Sm.+L10 Fied.+L6 Fied.+L25 HFER & 91.3\% & 189/207 & 38.3\% \\
\bottomrule
\end{tabular}
\end{center}
\end{table}

\newpage

\subsection{Llama 3.1 8B}

\begin{table}[h]
\caption{Llama 3.1 8B recall optimized on Finance ($N=1000$, 613 hallucinations, 61.3\% rate).}
\begin{center}
\footnotesize
\begin{tabular}{lccc}
\toprule
Features & Recall & Detected & Precision \\
\midrule
L27 Entropy & 95.8\% & 587/613 & 64.4\% \\
L28 Fiedler+L30 Fiedler & 89.7\% & 550/613 & 69.4\% \\
L30 Fied.+L1 Ent.+L25 HFER & 86.1\% & 528/613 & 81.0\% \\
L29 Fied.+L30 Sm.+L28 Fied.+L30 Fied.+L19 HFER & 88.7\% & 544/613 & 79.6\% \\
\bottomrule
\end{tabular}
\end{center}
\end{table}

\section{Theoretical Background}
\label{app:theory}

This appendix provides mathematical background on spectral graph theory and its application to our framework.

\subsection{Graph Laplacians}
\label{app:theory:laplacians}

\begin{definition}[Combinatorial Laplacian]
For an undirected weighted graph $\mathcal{G} = (\mathcal{V}, \mathcal{E}, \bm{W})$ with $N$ vertices and symmetric weight matrix $\bm{W} \in \R^{N \times N}_{\geq 0}$, the combinatorial graph Laplacian is:
\begin{equation}
\bm{L} = \bm{D} - \bm{W}
\end{equation}
where $\bm{D} = \diag(d_1, \ldots, d_N)$ with $d_i = \sum_{j=1}^N W_{ij}$ is the degree matrix.
\end{definition}

\begin{definition}[Normalized Laplacians]
Two normalized variants are commonly used:
\begin{align}
\bm{L}_{\text{sym}} &= \bm{D}^{-1/2} \bm{L} \bm{D}^{-1/2} = \bm{I} - \bm{D}^{-1/2} \bm{W} \bm{D}^{-1/2} \\
\bm{L}_{\text{rw}} &= \bm{D}^{-1} \bm{L} = \bm{I} - \bm{D}^{-1} \bm{W}
\end{align}
The symmetric normalized Laplacian $\bm{L}_{\text{sym}}$ has eigenvalues in $[0, 2]$ for connected graphs; the random walk Laplacian $\bm{L}_{\text{rw}}$ relates to diffusion processes on graphs.
\end{definition}

\begin{proposition}[Laplacian Properties]
\label{prop:laplacian_properties}
The combinatorial Laplacian $\bm{L}$ satisfies:
\begin{enumerate}[label=(\roman*)]
\item $\bm{L}$ is symmetric positive semidefinite.
\item $\bm{L} \bm{1} = \bm{0}$, so $\lambda_1 = 0$ with eigenvector $\bm{1}$.
\item For any $\bm{x} \in \R^N$: $\bm{x}^\top \bm{L} \bm{x} = \frac{1}{2} \sum_{i,j} W_{ij} (x_i - x_j)^2 \geq 0$.
\item The multiplicity of eigenvalue 0 equals the number of connected components.
\item Eigenvalues satisfy $0 = \lambda_1 \leq \lambda_2 \leq \cdots \leq \lambda_N \leq 2 d_{\max}$.
\end{enumerate}
\end{proposition}

\begin{proof}
(i) Symmetry follows from $\bm{W} = \bm{W}^\top$ and $\bm{D}$ diagonal. Positive semidefiniteness follows from (iii).

(ii) $(\bm{L} \bm{1})_i = d_i \cdot 1 - \sum_j W_{ij} \cdot 1 = d_i - d_i = 0$.

(iii) Direct computation:
\begin{align}
\bm{x}^\top \bm{L} \bm{x} &= \bm{x}^\top \bm{D} \bm{x} - \bm{x}^\top \bm{W} \bm{x} = \sum_i d_i x_i^2 - \sum_{i,j} W_{ij} x_i x_j \\
&= \frac{1}{2} \left( \sum_i d_i x_i^2 - 2\sum_{i,j} W_{ij} x_i x_j + \sum_j d_j x_j^2 \right) \\
&= \frac{1}{2} \sum_{i,j} W_{ij} (x_i^2 - 2x_i x_j + x_j^2) = \frac{1}{2} \sum_{i,j} W_{ij} (x_i - x_j)^2
\end{align}

(iv) If $\mathcal{G}$ has $k$ connected components, we can construct $k$ linearly independent vectors constant on each component, all in the null space of $\bm{L}$.

(v) The lower bound follows from (i). For the upper bound, using Gershgorin's theorem: eigenvalues lie in $\bigcup_i [d_i - \sum_{j \neq i} |L_{ij}|, d_i + \sum_{j \neq i} |L_{ij}|] = \bigcup_i [0, 2d_i] \subseteq [0, 2d_{\max}]$.
\end{proof}

\subsection{Algebraic Connectivity}
\label{app:theory:fiedler}

\begin{definition}[Fiedler Value and Vector]
The algebraic connectivity or Fiedler value of a connected graph is $\lambda_2(\bm{L})$, the second-smallest Laplacian eigenvalue. The corresponding eigenvector is the Fiedler vector.
\end{definition}

\begin{theorem}[Fiedler, 1973]
\label{thm:fiedler}
For a connected graph $\mathcal{G}$:
\begin{enumerate}[label=(\roman*)]
\item $\lambda_2 > 0$ if and only if $\mathcal{G}$ is connected.
\item $\lambda_2 \leq \kappa(\mathcal{G})$ where $\kappa(\mathcal{G})$ is the vertex connectivity.
\item The Fiedler vector provides a graph embedding useful for partitioning.
\end{enumerate}
\end{theorem}

\begin{theorem}[Cheeger Inequality]
\label{thm:cheeger}
The Fiedler value relates to the Cheeger constant (isoperimetric number) $h(\mathcal{G})$:
\begin{equation}
\frac{\lambda_2}{2} \leq h(\mathcal{G}) \leq \sqrt{2 \lambda_2}
\end{equation}
where $h(\mathcal{G}) = \min_{S \subset \mathcal{V}, |S| \leq N/2} \frac{|\partial S|}{\min(|S|, |\bar{S}|)}$ and $|\partial S| = \sum_{i \in S, j \notin S} W_{ij}$.
\end{theorem}

The Cheeger inequality establishes that $\lambda_2$ characterizes the ``bottleneck'' of information flow in the graph: high $\lambda_2$ implies no sparse cuts, enabling efficient global communication. Low $\lambda_2$ indicates sparse cuts that impede global communication---a signature we observe in hallucinating models where attention fails to maintain coherent global structure.

\subsection{Graph Signal Processing}
\label{app:theory:gsp}

\begin{definition}[Graph Signal]
A graph signal is a function $f: \mathcal{V} \to \R$ assigning a real value to each vertex, representable as a vector $\bm{f} \in \R^N$.
\end{definition}

\begin{definition}[Graph Fourier Transform]
The Graph Fourier Transform (GFT) of a signal $\bm{f}$ with respect to Laplacian $\bm{L} = \bm{U} \bm{\Lambda} \bm{U}^\top$ is:
\begin{equation}
\hat{\bm{f}} = \bm{U}^\top \bm{f}
\end{equation}
The component $\hat{f}_k = \bm{u}_k^\top \bm{f}$ represents the signal's content at graph frequency $\lambda_k$.
\end{definition}

\begin{proposition}[Frequency Interpretation]
\label{prop:frequency}
The Laplacian eigenvectors provide a notion of frequency on graphs:
\begin{enumerate}[label=(\roman*)]
\item $\bm{u}_1 = \bm{1}/\sqrt{N}$ is the ``DC component'' (constant signal).
\item Eigenvectors $\bm{u}_k$ with small $\lambda_k$ vary slowly across edges.
\item Eigenvectors with large $\lambda_k$ vary rapidly, changing sign frequently.
\item $\bm{f}^\top \bm{L} \bm{f} = \sum_k \lambda_k \hat{f}_k^2$: the Dirichlet energy is a weighted sum of squared spectral coefficients, with high-frequency components penalized more.
\end{enumerate}
\end{proposition}

\subsection{Spectral Entropy Interpretation}
\label{app:theory:entropy}

The Spectral Entropy $\SE^{(\ell)}$ measures the uniformity of the spectral energy distribution. When $\SE$ is low, energy concentrates in few eigenmodes, indicating structured, coherent signals. When $\SE$ is high, energy disperses across many modes, indicating noise or incoherence.

For a signal $\bm{X}$ that is perfectly aligned with a single eigenmode $\bm{u}_k$, we have $\SE = 0$. For a signal with equal energy in all modes, $\SE = \log N$. Hallucinations appear to push the attention spectrum toward higher entropy states, dispersing the coherent structure that characterizes valid reasoning.

This thermodynamic interpretation suggests that hallucination is not merely selecting incorrect tokens but represents a phase transition in the model's internal state from ordered (coherent attention) to disordered (entropic attention).

\subsection{Thermodynamic Interpretation of the Loud Liar Phenomenon}
\label{app:theory:loudliar}

The dramatic difference in detectability between architectures (98.2\% for Llama vs 78.0\% for Qwen single-feature recall) suggests fundamentally different failure geometries. We hypothesize that larger models develop more specialized attention patterns during training, leading to more catastrophic failures when these patterns break down.

When Llama hallucinates, its attention mechanism appears to undergo a phase transition: the highly structured patterns required for valid tool generation collapse into high-entropy noise. This ``loud'' failure mode produces unmistakable spectral signatures. Smaller models like Qwen may fail more ``quietly,'' with hallucinations producing spectral patterns that overlap more with valid outputs.

This interpretation aligns with recent work on neural network thermodynamics and suggests that spectral guardrails are most effective precisely where they are most needed: on larger, more capable models whose failures are geometrically isolated.

\section{Implementation Details}
\label{app:implementation}

\subsection{Attention Extraction}
\label{app:implementation:attention}

We extract attention matrices using standard hooks into the transformer forward pass. For each model:

\begin{itemize}[leftmargin=*,topsep=2pt,itemsep=1pt]
\item \textbf{Llama 3.1 8B:} We access \texttt{model.layers[l].self\_attn.attn\_weights} post-softmax.
\item \textbf{Qwen 2.5 0.5B:} Similar structure via \texttt{model.layers[l].attn.attention\_weights}.
\item \textbf{Mistral 7B:} Via \texttt{model.layers[l].self\_attn.attn\_weights}. Note that sliding window attention (SWA) produces sparse attention matrices with non-zero entries only within the sliding window.
\end{itemize}

All models use \texttt{output\_attentions=True} during forward passes.

\subsection{Spectral Computation}
\label{app:implementation:spectral}

\begin{algorithm}[h]
\caption{Spectral Diagnostic Extraction}
\label{alg:spectral}
\begin{algorithmic}[1]
\REQUIRE Attention matrices $\{\bm{A}^{(\ell,h)}\}_{\ell,h}$, hidden states $\{\bm{X}^{(\ell)}\}_\ell$
\ENSURE Spectral diagnostics $\{\lambda_2^{(\ell)}, \HFER^{(\ell)}, \mathcal{S}^{(\ell)}, \SE^{(\ell)}\}_\ell$
\FOR{layer $\ell = 1$ to $L$}
    \STATE // Symmetrize and aggregate attention
    \FOR{head $h = 1$ to $H$}
        \STATE $\bm{W}^{(\ell,h)} \gets \frac{1}{2}(\bm{A}^{(\ell,h)} + (\bm{A}^{(\ell,h)})^\top)$
        \STATE $s_h \gets \sum_{i,j} A^{(\ell,h)}_{ij}$
    \ENDFOR
    \STATE $\alpha_h \gets s_h / \sum_g s_g$ for all $h$
    \STATE $\bar{\bm{W}}^{(\ell)} \gets \sum_h \alpha_h \bm{W}^{(\ell,h)}$
    \STATE // Compute Laplacian
    \STATE $\bar{\bm{D}}^{(\ell)} \gets \diag(\bar{\bm{W}}^{(\ell)} \bm{1})$
    \STATE $\bm{L}^{(\ell)} \gets \bar{\bm{D}}^{(\ell)} - \bar{\bm{W}}^{(\ell)}$
    \STATE // Eigendecomposition
    \STATE $\bm{\Lambda}^{(\ell)}, \bm{U}^{(\ell)} \gets \text{eig}(\bm{L}^{(\ell)})$ \COMMENT{Sorted by eigenvalue}
    \STATE // Extract diagnostics
    \STATE $\lambda_2^{(\ell)} \gets \Lambda^{(\ell)}_{2,2}$ \COMMENT{Fiedler value}
    \STATE $\hat{\bm{X}}^{(\ell)} \gets (\bm{U}^{(\ell)})^\top \bm{X}^{(\ell)}$ \COMMENT{Graph Fourier Transform}
    \STATE $e_m \gets \|\hat{\bm{X}}^{(\ell)}_{m,\cdot}\|_2^2$ for $m = 1, \ldots, N$
    \STATE $\HFER^{(\ell)} \gets \sum_{m > N/2} e_m / \sum_m e_m$ \COMMENT{High-freq energy ratio}
    \STATE $\mathcal{S}^{(\ell)} \gets 1 - \Tr((\bm{X}^{(\ell)})^\top \bm{L}^{(\ell)} \bm{X}^{(\ell)}) / (\lambda_N^{(\ell)} \|\bm{X}^{(\ell)}\|_F^2)$ \COMMENT{Smoothness}
    \STATE $p_m \gets e_m / \sum_r e_r$; $\SE^{(\ell)} \gets -\sum_m p_m \log p_m$ \COMMENT{Spectral Entropy}
\ENDFOR
\end{algorithmic}
\end{algorithm}

\subsection{Computational Complexity}
\label{app:implementation:complexity}

\begin{table}[h]
\centering
\begin{tabular}{@{}lc@{}}
\toprule
\textbf{Operation} & \textbf{Complexity} \\
\midrule
Symmetrization (per head) & $O(N^2)$ \\
Aggregation (all heads) & $O(H N^2)$ \\
Laplacian construction & $O(N^2)$ \\
Full eigendecomposition & $O(N^3)$ \\
Partial eigendecomposition ($k$ eigenvalues) & $O(N^2 k)$ \\
Graph Fourier Transform & $O(N^2 d)$ \\
Diagnostic computation & $O(N d)$ \\
\midrule
\textbf{Total (per layer)} & $O(N^3 + N^2 d)$ \\
\bottomrule
\end{tabular}
\caption{Computational complexity of spectral analysis. For typical tool call lengths ($N < 200$) and hidden dimensions ($d \approx 4096$), this adds $<$5\% overhead to inference.}
\label{tab:complexity}
\end{table}

\section{Per-Metric Detection Performance}

Table \ref{tab:comprehensive_metrics} presents detection performance for each spectral metric independently, using the best layer for each metric-model combination. This analysis reveals which metrics are most diagnostic for each architecture and condition.

\begin{table}[h!]
\caption{Comprehensive per-metric detection performance (best layer per metric). Smoothness achieves exceptionally strong effect sizes ($d > 1.4$) on Qwen. Llama General is the most balanced performer across metrics. Detection capability degrades with increasing temperature.}
\label{tab:comprehensive_metrics}
\begin{center}
\footnotesize
\setlength{\tabcolsep}{3pt}
\begin{tabular}{lcccccccc}
\toprule
& \multicolumn{2}{c}{HFER} & \multicolumn{2}{c}{Fiedler} & \multicolumn{2}{c}{Smoothness} & \multicolumn{2}{c}{Entropy} \\
\cmidrule(lr){2-3} \cmidrule(lr){4-5} \cmidrule(lr){6-7} \cmidrule(lr){8-9}
Model / Condition & AUC & $d$ & AUC & $d$ & AUC & $d$ & AUC & $d$ \\
\midrule
\multicolumn{9}{l}{\textit{Cross-Model Comparison (T=0.3)}} \\
\midrule
Llama 3.1 8B (General) & 0.774 & 1.152 & 0.777 & 1.064 & 0.761 & 1.006 & 0.767 & 1.018 \\
Mistral 7B (General) & 0.754 & 0.984 & 0.801 & 1.108 & \textbf{0.804} & \textbf{1.291} & 0.782 & 1.030 \\
Qwen 2.5 0.5B (General) & 0.705 & 0.778 & 0.657 & 0.749 & \textbf{0.812} & \textbf{1.448} & 0.740 & 0.851 \\
\midrule
\multicolumn{9}{l}{\textit{Domain-Specific (T=0.3)}} \\
\midrule
Llama 3.1 8B (Finance) & 0.649 & 0.486 & 0.696 & 0.662 & 0.686 & 0.606 & 0.667 & 0.566 \\
Qwen 2.5 0.5B (Finance) & 0.697 & 0.806 & 0.670 & 0.997 & \textbf{0.816} & \textbf{1.468} & 0.729 & 0.842 \\
\midrule
\multicolumn{9}{l}{\textit{Temperature Ablation (Qwen 2.5 0.5B, General)}} \\
\midrule
T=0.3 & 0.705 & 0.778 & 0.657 & 0.749 & \textbf{0.812} & \textbf{1.448} & 0.740 & 0.851 \\
T=0.7 & 0.711 & 0.765 & 0.724 & 1.020 & \textbf{0.800} & \textbf{1.482} & 0.751 & 0.828 \\
T=1.0 & 0.668 & 0.729 & 0.682 & 0.830 & 0.748 & 0.743 & 0.652 & 0.502 \\
T=1.5 & 0.641 & 0.505 & 0.747 & 0.856 & 0.667 & 0.667 & 0.658 & 0.607 \\
\midrule
\multicolumn{9}{l}{\textit{Temperature Ablation (Llama 3.1 8B, Finance)}} \\
\midrule
T=0.3 & 0.649 & 0.486 & 0.696 & 0.662 & 0.686 & 0.606 & 0.667 & 0.566 \\
T=1.0 & 0.631 & 0.439 & 0.674 & 0.680 & 0.650 & 0.495 & 0.691 & 0.677 \\
\bottomrule
\end{tabular}
\end{center}
\end{table}

\paragraph{Key Observations.}
\begin{itemize}[leftmargin=*,topsep=0pt,itemsep=2pt]
    \item \textbf{Smoothness dominates for Qwen}: Across both Finance and General domains at $T=0.3$ and $T=0.7$, Smoothness achieves $d > 1.4$, far exceeding other metrics. This suggests Qwen's hallucinations manifest primarily as smoothness collapse rather than entropic dispersal.
    \item \textbf{Llama General is balanced}: On the General domain, Llama shows strong and consistent performance across all four metrics ($d \approx 1.0$--$1.15$), with no single metric dominating. This balanced profile may explain why multi-feature combinations are effective.
    \item \textbf{Temperature degrades detection}: As temperature increases from $T=0.3$ to $T=1.5$ on Qwen, HFER AUC drops from 0.705 to 0.641, and Smoothness $d$ drops from 1.448 to 0.667. Higher temperature adds entropic noise that masks the spectral signature of hallucination. Notably, Fiedler remains somewhat robust at high temperature ($d = 0.856$ at $T=1.5$).
    \item \textbf{Domain affects Llama more than Qwen}: Llama's effect sizes drop substantially from General ($d \approx 1.0$) to Finance ($d \approx 0.5$--$0.6$), while Qwen maintains strong Smoothness signal across domains.
\end{itemize}

\section{Failure of Standard Confidence Metrics}
\label{app:confidence_failure}

We evaluated standard uncertainty metrics, Perplexity (PPL) and Mean Log Probability, across all datasets. As shown in Table \ref{tab:ppl_baselines}, these metrics exhibit dangerous inconsistency.

While Perplexity provides a reasonable signal for Mistral 7B (AUC 0.79), it collapses for the Qwen 2.5 family, yielding an AUC of $\approx 0.27$. One might notice that PPL achieves near-perfect recall (99-100\%) on Qwen. This is illusory: as indicated by the precision ($\approx 17\%$, matching the base rate), the metric is effectively classifying every sample as a hallucination to achieve this recall. It functions as a trivial detector with no discriminative power.

Furthermore, the directional correlation of LogProb shifts across domains (positive for Llama Finance, negative for Qwen), whereas Spectral Entropy provides a consistent thermodynamic signature (higher entropy = hallucination) across all tested models.

\begin{table}[h]
\caption{Baseline Performance of Probability-Based Metrics. Qwen 2.5 exhibits a negative correlation (AUC 0.27). Note that while Recall is high (0.99), Precision collapses to the base rate ($\approx 0.17$), indicating the metric is simply flagging all samples rather than discriminating.}
\label{tab:ppl_baselines}
\begin{center}
\begin{small}
\begin{tabular}{lcccccc}
\toprule
& \multicolumn{3}{c}{\textbf{Perplexity (PPL)}} & \multicolumn{3}{c}{\textbf{Mean LogProb}} \\
\cmidrule(lr){2-4} \cmidrule(lr){5-7}
Dataset / Model & AUC & Prec & Rec & AUC & Prec & Rec \\
\midrule
\textit{Llama 3.1 8B (Fin, $T=1.0$)} & 0.64 & 0.59 & 0.89 & 0.52 & 0.52 & 0.99 \\
\textit{Llama 3.1 8B (Fin, $T=0.3$)} & 0.67 & 0.65 & \textbf{0.98} & 0.55 & 0.62 & 0.99 \\
\textit{Llama 3.1 8B (Gen)} & 0.71 & 0.34 & 0.72 & 0.54 & 0.24 & 0.76 \\
\midrule
\textit{Mistral 7B v0.1 (Gen)} & \textbf{0.79} & 0.49 & 0.59 & 0.61 & 0.27 & 0.85 \\
\midrule
\textit{Qwen 2.5 0.5B (Fin)} & \textcolor{red}{0.27} & 0.17 & 1.00 & 0.53 & 0.17 & 0.99 \\
\textit{Qwen 2.5 0.5B (Gen)} & \textcolor{red}{0.28} & 0.20 & 0.99 & 0.49 & 0.20 & 1.00 \\
\bottomrule
\end{tabular}
\end{small}
\end{center}
\end{table}

\section{Reproducibility}

\paragraph{Code.} 
To facilitate reproducibility, the complete implementation including attention extraction and spectral computation is available at \url{https://github.com/vcnoel/spectral-tool-use}.

\paragraph{Data.} Glaive Function Calling v2 dataset \citet{glaive2024function}.

\paragraph{Models.} All models accessed via Hugging Face:
\begin{itemize}[leftmargin=*,topsep=0pt,itemsep=2pt]
    \item Qwen/Qwen2.5-0.5B-Instruct
    \item mistralai/Mistral-7B-Instruct-v0.1
    \item meta-llama/Llama-3.1-8B-Instruct
\end{itemize}

\paragraph{Hardware and Runtime.}
\begin{itemize}[leftmargin=*,topsep=0pt,itemsep=2pt]
    \item Hardware: Single NVIDIA A100 GPU (40GB)
    \item Spectral computation overhead: 10--50ms per tool call (for $N < 200$ tokens)
    \item Total runtime for 1000 samples: 15--30 minutes per model
    \item Total compute time: approximately 4 hours per model for full evaluation
\end{itemize}

\paragraph{Hyperparameters.} Temperature $T=0.3$ for all generation unless otherwise noted in ablation studies. Thresholds calibrated via grid search on held out validation set (50-100 samples).

\end{document}